\documentclass[accepted]{uai2021} 

\usepackage[american]{babel}

\usepackage{natbib} 
\usepackage{mathtools} 
\usepackage{booktabs} 
\usepackage{tikz} 

\usepackage{graphicx}
\usepackage{amsmath}
\usepackage{amsthm}
\usepackage{amssymb}
\usepackage{bm}

\usepackage{relsize} 

\usepackage{eucal} 

\usepackage{algorithm}
\usepackage[noend]{algorithmic}

\usepackage{wrapfig}
\usepackage{subcaption}

\usetikzlibrary{matrix}
\usepackage{pgfplots}
\usepgfplotslibrary{fillbetween}
\pgfplotsset{compat=1.16,}
\usetikzlibrary{patterns}
\newcommand{\symfootnote}[1]{%
\let\oldthefootnote=\thefootnote%
\stepcounter{mpfootnote}%
\addtocounter{footnote}{-1}%
\renewcommand{\thefootnote}{\fnsymbol{mpfootnote}}%
\footnote{#1}%
\let\thefootnote=\oldthefootnote%
}

\newtheorem{theorem}{Theorem}

\newtheorem{lemma}[theorem]{Lemma}

\newtheorem{cond}{Condition}

\newtheorem{proposition}{Proposition}
\theoremstyle{definition}
\newtheorem{definition}{Definition}

\renewcommand\vec{\mathbf}

\renewcommand{\epsilon}{\varepsilon}

\newcommand{\reward}{\textit{r}}
\newcommand{\payoff}{\textit{payoff}}
\newcommand{\regret}{\textit{regret}}
\newcommand{\simplex}{\Delta}
\newcommand{\defpolicy}{\pi}
\newcommand{\optpolicy}{\defpolicy^\star}
\newcommand{\altpolicy}{\hat{\defpolicy}}
\newcommand{\policyset}{\Pi}
\newcommand{\mixedpolicy}{\tilde{\defpolicy}}
\newcommand{\mixedoptpolicy}{\tilde{\defpolicy}^\star}
\newcommand{\attract}{\mathbf{z}} 
\newcommand{\attractset}{Z}
\newcommand{\mixedattract}{\tilde{\attract}}
\newcommand{\attracti}[1]{z_{#1}}

\newcommand{\sleep}{\kappa} 

\newcommand{\effort}{\mathbf{a}}
\newcommand{\efforti}[1]{a_{#1}}
\newcommand{\attractlow}[1]{\underline{\attracti{#1}}}
\newcommand{\attracthigh}[1]{\overline{\attracti{#1}}}
\newcommand{\wildlife}{\mathbf{w}}
\newcommand{\wildlifei}[1]{w_{#1}}
\newcommand{\attacki}[1]{k_{#1}} 

\newcommand{\curreffortcoef}{\gamma}
\newcommand{\pasteffortcoef}{\beta}
\newcommand{\neighborcoef}{\eta}
\newcommand{\wildlifegrowth}{\psi}

\newcommand{\state}{\mathbf{s}}

\title{Robust Reinforcement Learning Under Minimax Regret for Green Security}

\author[1]{\href{mailto:Lily Xu <lily_xu@g.harvard.edu>?Subject=UAI 2021 paper}{Lily~Xu}{}}
\author[1]{Andrew~Perrault}
\author[2]{Fei~Fang}
\author[1]{Haipeng~Chen}
\author[1]{Milind Tambe}
\affil[1]{%
    Center for Research on Computation and Society\\
    Harvard University
}
\affil[2]{%
    Institute for Software Research\\
    Carnegie Mellon University
}

\begin{document}
\maketitle

\begin{abstract}
Green security domains feature defenders who plan patrols in the face of uncertainty about the adversarial behavior of poachers, illegal loggers, and illegal fishers. Importantly, the deterrence effect of patrols on adversaries' future behavior makes patrol planning a sequential decision-making problem. 
Therefore, we focus on robust sequential patrol planning for green security following the minimax regret criterion, which has not been considered in the literature. We formulate the problem as a game between the defender and nature who controls the parameter values of the adversarial behavior and design an algorithm MIRROR to find a robust policy. MIRROR uses two reinforcement learning--based oracles and solves a restricted game considering limited defender strategies and parameter values.
We evaluate MIRROR on real-world poaching data.
\end{abstract}



\section{Introduction}
\label{sec:intro}



Defenders in green security domains aim to protect wildlife, forests, and fisheries and are tasked to strategically allocate limited resources in a partially unknown environment \citep{fang2015security}. For example, to prevent poaching, rangers will patrol a protected area to locate and remove snares (Figure~\ref{fig:rangers}). 
Over the past few years, predictive models of poacher behavior have been developed and deployed to parks around the world, creating both opportunity and urgency for effective patrol planning strategies~\citep{kar2017cloudy,gurumurthy2018exploiting,xu2020stay}.

While patrol planning for security has been studied under game-theoretic frameworks~\citep{korzhyk2010complexity,basilico2012patrolling,marecki2012playing}, green security domains have two crucial challenges: \emph{uncertainty} in adversaries' behavior model and the \emph{deterrence} effect of patrols --- how current patrols reduce the likelihood that adversaries attack in the \emph{future}. Data is often scarce in these domains and it is hard to learn an accurate adversarial behavior model~\citep{fang2015security,xu2016playing,sessa2020learning}; patrols planned without considering the imperfection of the behavior model would have limited effectiveness in practice.
Deterrence is hypothesized to be a primary mechanism through which patrols reduce illegal activity~\citep{levitt1998increased}, especially in domains such as wildlife protection, as rangers rarely apprehend poachers and only remove an estimated 10\% of snares \citep{moore2018ranger}. 
These characteristics make apparent the need for robust sequential patrol planning for green security, which is the focus of this paper. We confirm the deterrence effect in green security domains for the first time through analyzing real poaching data, providing real-world footing for this research.

\begin{figure}
  \centering
  \includegraphics[width=0.8\columnwidth]{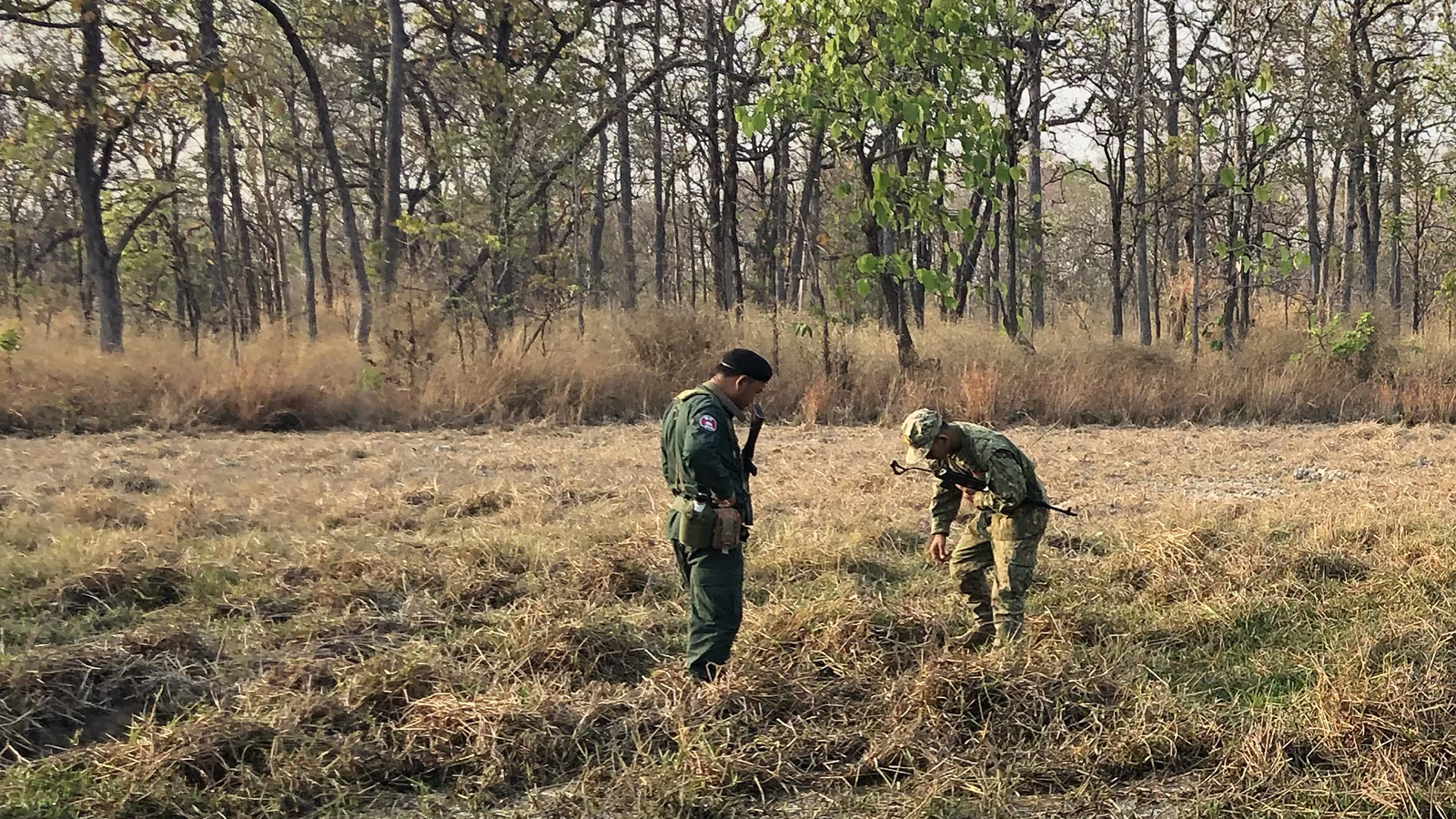}
  \caption{Rangers remove a snare in Srepok Wildlife Sanctuary in Cambodia, where the Cambodian government plans to reintroduce tigers in 2022.}
  \label{fig:rangers}
\end{figure}





In this paper, we consider the \emph{minimax regret} criterion for robustness~\citep{savage1951theory,wang2003incremental}: minimize the maximum regret, which is defined as the maximum difference under any uncertainty instantiation between the expected return of the chosen strategy against the expected return of an optimal strategy. 
Compared to maximin reward, minimax regret is more psychologically grounded according to phenomena such as risk aversion \citep{loomes1982regret} and is less conservative and sensitive to worst-case outcomes \citep{kouvelis2013robust}. However, optimizing for regret is challenging \citep{nguyen2014regret}, especially for complex sequential decision making problems as evidenced by lack of past work on minimax regret in deep reinforcement learning (RL), despite the success and popularity of deep RL in recent years~\citep{mnih2015human,lillicrap2015continuous}. 
The main obstacle is that when the environment parameters change, the reward of a strategy changes \emph{and} there may be a new optimal strategy, making it hard to quickly estimate the maximum regret of a strategy.
We overcome this obstacle by developing a new method named MIRROR\symfootnote{Code available at \href{https://github.com/lily-x/mirror}{https://github.com/lily-x/mirror}} that enables minimax regret planning under environment uncertainty using RL. 
We model the robust planning problem as a two-player, zero-sum game between an agent, who looks for minimax regret--optimal policies, and nature, who looks for regret-maximizing instantiations of the uncertain environment parameters (we refer to this game as a \emph{max-regret game}). 
This model enables us to use the double oracle method~\citep{mcmahan2003planning} and the policy-space response oracle (PSRO) framework~\citep{lanctot2017unified} to incrementally generate strategies and environment parameters to be considered. More specifically, MIRROR includes two RL-based oracles. The agent oracle solves a typical sequential decision-making problem and returns a defender strategy. The nature oracle finds the environment parameters and the corresponding optimal defender strategy that lead to the highest regret for a given defender strategy. We use a policy-gradient approach for both oracles. In the nature oracle, we treat the environment parameters as input to the policy network and update the environment parameters and the network parameters with a wake--sleep procedure. We further enhance the algorithm with parameter perturbation in both oracles.

Our contributions are summarized as follows. 
(1)~We provide a realistic adversary model learned from real-world poaching data from Queen Elizabeth National Park (QENP) in Uganda, which demonstrates deterrence and opens the door to further RL research in service of protecting the environment.
(2)~We propose MIRROR, a framework to calculate minimax regret--optimal policies using RL for the first time, and apply this approach to green security domains.
(3)~We prove that MIRROR converges to an $\epsilon$--optimal strategy in a finite number of iterations in our green security setting. 
(4)~We empirically evaluate MIRROR on real-world poaching data from QENP.

\section{Related Work}
\label{sec:related-work}




\paragraph{Robust planning with minimax regret}
Minimax regret has been considered for preference elicitation of additive utilities \citep{braziunas2007minimax} and rewards \citep{regan2009regret}, as well as robotics planning in uncertain Markov decision processes with a model-based approach \citep{rigter2021minimax}. 
Double oracle~\citep{mcmahan2003planning} has been used to optimize for minimax regret in security games and in robust optimization~\citep{nguyen2014regret,gilbert2017double} but in single-action settings, not policy spaces. Double oracle has also been used without minimax regret for solving large zero-sum games~\citep{bosansky2014exact,jain2011double}.

\paragraph{Robust planning in RL}
Robustness in RL has been heavily studied, both in the context of robust adversarial RL \citep{pinto2017robust,pan2019risk,zhang2020robustdeep} and nonstationarity in multi-agent RL settings \citep{li2019robust,zhang2020robustmulti}. For example, PSRO extends double oracle from state-independent pure strategies to policy-space strategies to be used for multiplayer competitive games \citep{lanctot2017unified}. \citet{zhang2020robustdeep} consider robustness against adversarial perturbations on state observations.
The line of work whose setting is most similar to our problem is robust RL with model uncertainty, specifically in the transition and reward functions \citep{wang2020reinforcement,zhang2020robustmulti}. However, these approaches all consider robustness subject to maximin reward, whereas we optimize for minimax regret robustness. The two objectives are incompatible; we cannot simply substitute minimax regret into the reward function and solve using minimax reward, as computing the maximum regret incurs the challenge of knowing the optimal strategy and its corresponding reward.

\paragraph{Green security games (GSGs)} 
Literature on GSGs model the problem in green security domains as a game between a defender and boundedly rational attackers, with the assumption that attacker models can be learned from data~\citep{nguyen2016capture,yang2014adaptive,fang2016deploying,xu2017optimal}. Most of this work does not consider uncertainty in the learned attacker model and solve the patrol planning problem using mathematical programming, which is not scalable for planning sequential patrols over time horizons going beyond 2 to 3 timesteps. Past work addressing uncertainty in green security focuses on the setting with a stochastic adversary \citep{xu2021dual}, treating the problem as one of learning a good strategy against the optimal strategy in hindsight.
RL has been used for planning in GSGs with real-time information to model defenders responding to footprints during a patrol \citep{wang2019deep} and strategic signalling with drones \citep{venugopal2021reinforcement}.
However, uncertainty and robustness have not been explicitly considered together in GSG literature and much existing work on green security do not have access to real-world data and realistic models of deterrence.








\section{Problem Statement}
\label{sec:problem-statement}

In green security settings, we have a \emph{defender} (e.g., ranger) who conducts patrols in a protected area to prevent resource extraction by an \emph{attacker} (e.g., poacher or illegal logger). 
Let $N$ be the number of targets, such as $1 \times 1$~km regions in a protected area, that we are trying to protect. We have timesteps $t = 1, 2, \ldots, T$ up to some finite time horizon $T$ where each timestep represents, for example, a one-month period. The defender needs to choose a patrol strategy (also called the defender \emph{policy}) $\defpolicy \in \policyset$, which sequentially allocates patrol effort at each timestep. We denote patrol effort at time $t$ as $\effort^{(t)}$, where $\efforti{i}^{(t)} \in [0, 1]$ represents how much effort the patrollers allocate to target~$i$. We constrain total effort by a budget~$B$ such that $\sum_i \efforti{i}^{(t)} \leq B$ for all $t$. 

Consider the poaching scenario specifically. Let $\wildlife^{(t)} \in \mathbb{R}_{\geq 0}^N$ describe the distribution of wildlife in a protected area at timestep~$t$, with $\wildlifei{i}^{(t)}$ denoting wildlife density in target~$i$. What the rangers care about the most is the total wildlife density by the end of the planning horizon, i.e., $\sum_i w_i^{(T)}$. 
Threatening the wildlife population are poachers, who come into the park and place snares to trap animals. Their behavior is governed by a number of factors including the current patrol strategy, the past patrol strategy due to the deterrence effect, geographic features including distance from the park boundary, elevation, and land cover, and others. Lacking complete and high-quality data about past poaching patterns, we are not able to build an accurate model of poacher behavior. 

Therefore, we consider a parameterized model for attacker's behavior and assume that the values of some of the parameters, denoted by $\attract$, are uncertain. 
We assume that $\attract$ comes from a given uncertainty region $\attractset$, which is a compact set specifying a range $\attracti{j} \in [\attractlow{j}, \attracthigh{j}]$ for each uncertain parameter $j$. We have no a priori knowledge about distribution over $\attractset$. 
We want to plan a patrol strategy $\defpolicy$ for the defender that is robust to parameter uncertainty following the minimax regret criterion. Let $\reward(\defpolicy, \attract)$ be the defender's expected return for taking policy $\defpolicy$ under environment parameters $\attract$, e.g., the expected total wildlife density at the end of the planning horizon.
Then the regret incurred by the agent for playing strategy $\defpolicy$ when the parameter values are $\attract$ is $\regret(\defpolicy, \attract)=\reward(\optpolicy(\attract), \attract) - \reward(\defpolicy, \attract)$, where $\optpolicy(\attract)$ is the optimal policy that maximizes reward under parameters $\attract$.

Our objective is then to find a strategy $\defpolicy$ for the defender that minimizes maximum possible regret under any parameter values $\attract$ that falls within the uncertainty region $\attractset$.
Formally, we want to solve the following optimization problem
\begin{align}
    \label{eq:minimax-regret}
    \min_{\defpolicy} \max_{\attract} \left( \reward(\optpolicy(\attract), \attract) - \reward(\defpolicy, \attract) \right) \ .
\end{align}

We can formulate this robust planning problem as a two-player game between an \emph{agent} who wants to learn an optimal defender strategy (or policy) $\defpolicy$ against \emph{nature} who selects worst-case parameter values $\attract$. 
Then the agent's payoff is $-\regret(\defpolicy, \attract)$ and nature's payoff is $\regret(\defpolicy, \attract)$. 



\begin{definition}[Max-regret game]
\label{def:regret-game}
We define the \emph{max-regret game} as a zero-sum game between the agent and nature, where the agent's payoff is
\begin{align}
\label{eq:regret-game-payoff}
\resizebox{.9\linewidth}{!}{
    $\payoff(\defpolicy, \attract) = -\regret(\defpolicy, \attract) = \reward(\defpolicy, \attract) - \reward \left( \optpolicy(\attract), \attract \right) \ .$
}
\end{align}
\end{definition}
The agent can also choose a mixed strategy (or randomized policy) $\mixedpolicy$, which is a probability distribution over $\policyset$. We denote by $\simplex(\policyset)$ the set of the defender's mixed strategies. Likewise, we have mixed strategy $\mixedattract \in \simplex(\attractset)$ for nature.

\paragraph{Generalizability}

Our approach applies not just to green security domains, but is in fact applicable to any setting in which we must learn a sequential policy $\defpolicy$ with uncertainty in some environment parameters $\attract$ where our evaluation is based on minimax regret. 
Our framework is also not restricted to hyper-rectangular shaped uncertainty regions; any form of uncertainty with a compact set on which we do not have a prior belief would work. 





\subsection{Real-World Deterrence Model}
\label{sec:deterrence}

No previous work in artificial intelligence or conservation biology has provided evidence of deterrent effect of ranger patrols on poaching, a topic critically important to planning real-world ranger patrols. 
Thus in our work on planning for green security domains, we began by exploring an open question about how poachers respond to ranger patrols. 

Past work has investigated deterrence to inconclusive results \citep{ford2017real,dancer2019evaluation}. Using real poaching data from Queen Elizabeth National Park (QENP) in Uganda, we study the effect of patrol effort on poacher response. We find clear evidence of \emph{deterrence} in that higher levels of past patrols reduce the likelihood of poaching; we are the first to do so. 
We also find that more past patrols on neighboring targets increase the likelihood of poaching, suggesting \emph{displacement}. 

\begin{wrapfigure}[8]{r}{.4\columnwidth}
  \centering
  \includegraphics[width=\linewidth]{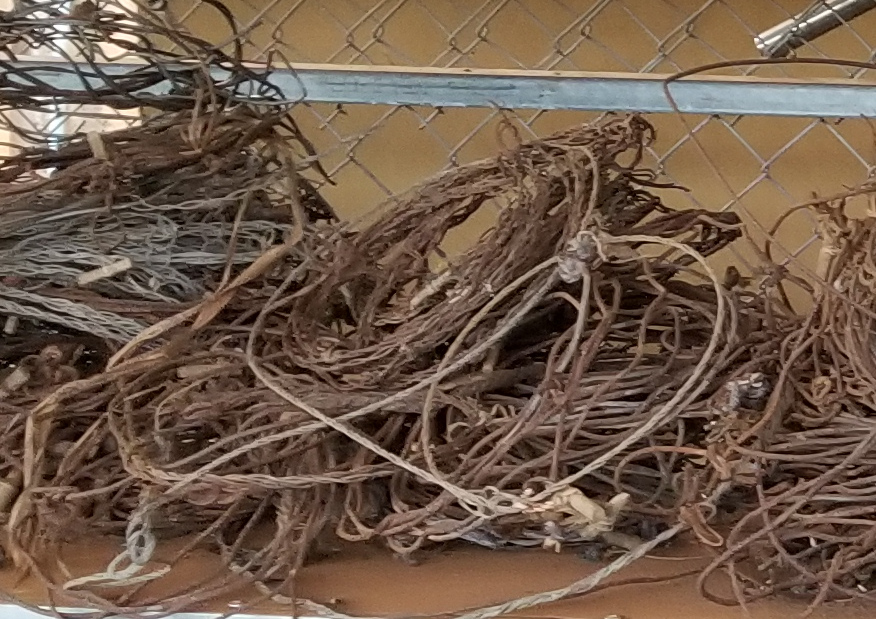}
  \caption{Snares.}
  \label{fig:snares}
\end{wrapfigure}
For each target, we calculate the total ranger patrol effort (in kilometers patrolled) and count the number of instances of illegal activity detected per month. We construct the patrol effort from 138,000 GPS waypoints across seven years of QENP poaching data. Observations of illegal activity are predominantly snares, but also include bullet cartridges, traditional weapons, and encounters with poachers. 


Let $\attracti{i}$ be the attractiveness of target~$i$ to poachers. To understand the effect of patrol effort on poaching activity, we learn the probability of detecting illegal activity in target~$i$ as a linear combination of
\begin{align}
    \attracti{i} + \curreffortcoef \cdot \efforti{i}^{(t)} + \pasteffortcoef \cdot \efforti{i}^{(t-1)},
\end{align}
which is then squashed through the logistic function. The parameter $\pasteffortcoef$ is the coefficient on past patrol effort $\efforti{i}^{(t-1)}$, measuring the deterrence effect we are trying to isolate, and $\curreffortcoef$ is the coefficient on current patrol effort $\efforti{i}^{(t)}$, measuring the difficulty of detecting snares. 

See Table~\ref{tab:coef-effort} for the learned values of the average attractiveness of each target $\overline{\attracti{i}}$, the coefficient on current effort $\curreffortcoef$, and the coefficient on past effort $\pasteffortcoef$. Each row studies this effect for a different time interval. For example, 1 year, 3 months looks at the impact of a year of previous patrol effort on illegal activity in the subsequent three months. The values for current and past patrol effort are normalized to highlight relatively high or low effort; $\curreffortcoef$ and $\pasteffortcoef$ reflect coefficients after normalization. The p-values for $\curreffortcoef$ and $\pasteffortcoef$ are all statistically significant with $p < 0.05$. 
The learned value of $\pasteffortcoef$ is negative across all datasets and settings --- \emph{thus, increased past patrol effort does have a measurable effect of deterring poaching.} 



\begin{table}
\centering
\caption{Learned coefficients, revealing deterrence} \label{tab:coef-effort} 
\begin{tabular}{ rccc } 
\toprule
& $\overline{\attracti{i}}$ & $\curreffortcoef$ & $\pasteffortcoef$ \\ 
\midrule
1 month, 1 month & $-9.285$ & $1.074$ & $-0.165$ \\ 
3 month, 3 month & $-10.624$ & $0.685$ & $-0.077$ \\ 
1 year, 1 month & $-9.287$ & $1.061$ & $-0.217$ \\ 
1 year, 3 month & $-10.629$ & $0.676$ & $-0.042$ \\ 
1 year, 1 year & $-8.559$ & $2.159$ & $-0.306$ \\ 
\bottomrule
\end{tabular}
\end{table}


Ideally, when poachers are deterred by ranger patrols, they would leave the park completely and desist their hunt of wildlife. Alternatively, they may move to other areas of the park. We show that the latter appears to be true. 
To do so, we study the spatial relationship between neighboring targets, using three spatial resolutions: $3 \times 3$, $5 \times 5$, and $7 \times 7$. We learn
\begin{align}
\resizebox{.9\linewidth}{!}{
$\attracti{i} + \curreffortcoef \cdot \efforti{i}^{(t)} + \pasteffortcoef \cdot \efforti{i}^{(t-1)}+ \neighborcoef \cdot \mathlarger{\sum}\limits_{j \in \text{neighbors}(i)} \efforti{j}^{(t-1)}$
}
\end{align}
where $\neighborcoef$ is the coefficient on past patrol effort on neighboring cells. As shown in Table~\ref{tab:coef-neighbors}, all learned values of $\neighborcoef$ are positive, indicating that increased patrols on neighboring areas increases the likelihood of poaching on a target in the next timestep. This result is consistent across the three spatial resolutions, and strongest for the narrowest window of $3 \times 3$. Observe as well that the values for $\overline{\attracti{i}}$, $\curreffortcoef$, and $\pasteffortcoef$ are remarkably consistent, demonstrating the robustness of our findings. 

\begin{table}
\centering
\caption{Learned coefficients, with neighbors included, revealing displacement at a 1-month interval}
\label{tab:coef-neighbors} 
\begin{tabular}{ rcccc }
\toprule
& $\overline{\attracti{i}}$ & $\curreffortcoef$ & $\pasteffortcoef$ & $\neighborcoef$ \\ 
\midrule
$3 \times 3$ & $-10.633$ & $0.687$ & $-0.098$ & $0.696$ \\ 
$5 \times 5$ & $-10.636$ & $0.688$ & $-0.097$ & $0.392$ \\ 
$7 \times 7$ & $-10.632$ & $0.688$ & $-0.097$ & $0.518$ \\ 
\bottomrule
\end{tabular}
\end{table}




\subsection{Green Security Model}
\label{sec:model}

In green security settings, the environment dynamics, including attacker behavior, can be described by an uncertain Markov decision process (UMDP) defined by the tuple $\langle \mathcal{S}, \state^{(0)}, \mathcal{A}, \mathcal{T}, \mathcal{R} \rangle$. The \emph{state} $\state$ is a tuple $(\effort^{(t-1)}, \wildlife^{(t-1)}, t)$ of past patrol effort, past wildlife, and current timestep with initial state $\state^{(0)} = (\vec{0}, \wildlife^{(0)}, 0)$. The \emph{action} $\effort^{(t)}$ is an effort vector describing time spent in each target, subject to a budget~$B$. Note that the model can be generalized to consider a \emph{sequence} of past effort and wildlife, which would model an attacker with a longer memory length. 

The environment dynamics are governed by the \emph{transitions}, a set $\mathcal{T}$ containing the possible mappings $\mathcal{T}_{\attract} : \mathcal{S} \mapsto \mathcal{S}$ where the transition $\mathcal{T}_{\attract} \in \mathcal{T}$ depends on environment parameters $\attract$. A mixed strategy $\mixedattract$ would produce a distribution over $\mathcal{T}$. These transitions are what makes our Markov decision process uncertain, as we do not know which mapping is the true transition. We model the adversary behavior with a simple logistic model, based on learned deterrence effect.
The probability that the poacher will attack a target~$i$ is given by the function
\begin{align}
\label{eq:poacher-prob}
\resizebox{.9\linewidth}{!}{
    $p_{i}^{(t)} = \text{logistic} \left( \attracti{i} + \pasteffortcoef \cdot \efforti{i}^{(t-1)} + \neighborcoef \cdot \mathlarger{\sum}\limits_{j \in \text{neighbors}(i)} \efforti{j}^{(t-1)} \right)$
}
\end{align}
where parameters $\pasteffortcoef < 0$ and $\neighborcoef > 0$ govern the strength of the deterrence and displacement effects, as described in Section~\ref{sec:deterrence}.
At each time step, the poacher takes some action $\attacki{i}^{(t)} \in \{0, 1\}$ where they either place a snare $\attacki{i}^{(t)} = 1$ or not $\attacki{i}^{(t)} = 0$. The realized adversary attack $k_i^{(t)}$ is drawn from Binomial distribution $\attacki{i}^{(t)} \sim B(p_i^{(t)})$.

The actions of the poacher and ranger then affect the wildlife population of the park. We use a regression model as in
\begin{align}
\label{eq:wildlife-response}
   \wildlifei{i}^{(t)} = \max\{0, (\wildlifei{i}^{(t-1)})^{\wildlifegrowth} - \alpha \cdot \attacki{i}^{(t-1)} \cdot (1 - \efforti{i}^{(t)})\}
\end{align}
where $\alpha > 0$ is the strength of poachers eliminating wildlife, and $\wildlifegrowth > 1$ is the wildlife natural growth rate. 

Our objective is to maximize the number of wildlife. The \emph{reward} $R$ is the sum of wildlife at the time horizon, so $R(\state^{(t)}) = \sum_{i=1}^{N} \wildlifei{i}^{(T)}$ if $t = T$ and $R(\state^{(t)})=0$ otherwise. To understand the relationship between defender return~$R$ in the game and the expected reward $\reward$ of the agent oracle from our objective in Equation~\ref{eq:minimax-regret}, we have
\begin{align}
    \label{eq:r}
    \reward(\defpolicy, \attract) = \mathbb{E}\left[R(\state^{(T)}) \right]
\end{align}
taking the expectation over states following the transition $\state^{(t+1)} \sim \mathcal{T}_{\attract}(\state^{(t)}, \defpolicy(\state^{(t)}), \state^{(t+1)})$ 
with initial state $\state^{(0)} = (\wildlife^{(0)}, \vec{0}, 0)$.




\section{Robust Planning}
\label{sec:robust-planning}

We propose MIRROR, which stands for MInimax Regret Robust ORacle. MIRROR is an algorithm for computing minimax regret--optimal policies in green security settings to plan patrols for a defender subject to uncertainty about the attackers' behavior. MIRROR also applies in generic RL contexts with a compact uncertainty set over transitions and rewards. 


To learn a minimax regret--optimal policy for the defender, we take an approach based on double oracle \citep{mcmahan2003planning}. 
Given our sequential problem setting of green security, we build on policy-space response oracle (PSRO) \citep{lanctot2017unified}. 
As discussed in Section~\ref{sec:problem-statement}, we pose the minimax regret optimization as a zero-sum game in the max regret space, between an agent (representing park rangers) who seeks to minimize max regret and nature (uncertainty over the adversary behavior parameters) which seeks to maximize regret. Our objective can be expressed as an optimization problem, as defined in Equation~\ref{eq:minimax-regret}.





\begin{figure*}
  \centering
  \includegraphics[width=.9\textwidth]{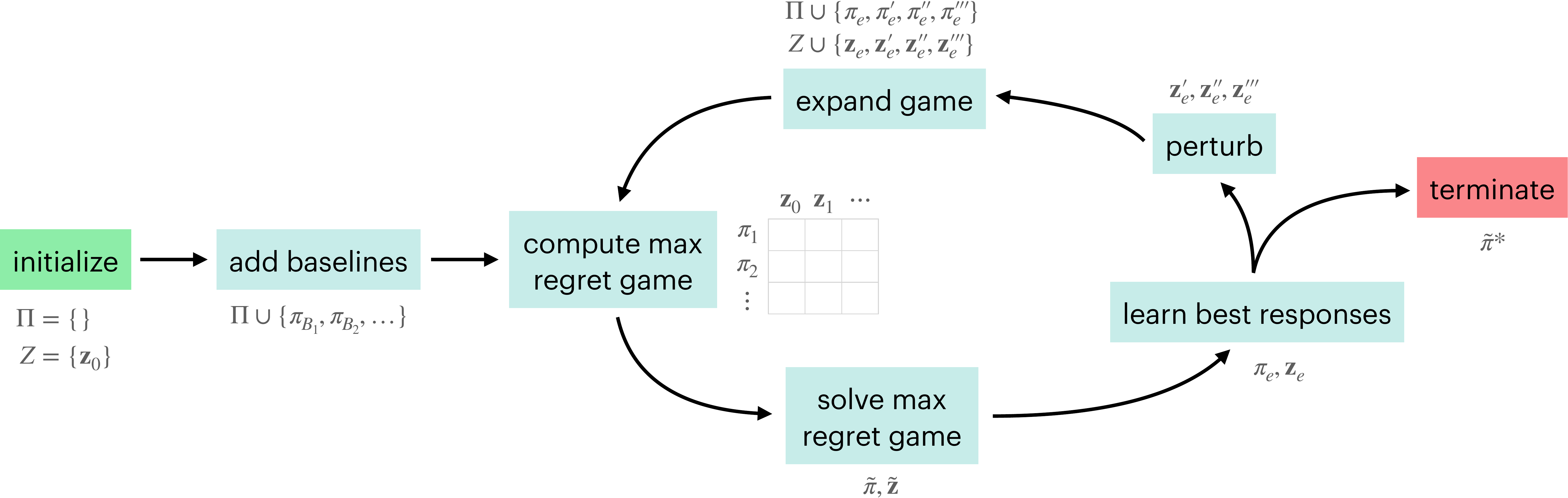}
  \caption{Our MIRROR algorithm, with figure design inspired by the double oracle figure from \citet{bovsansky2016algorithms}.}
  \label{fig:our-algorithm}
\end{figure*}

\begin{algorithm}[tb]
\caption{MIRROR: MInimax Regret Robust ORacle}
\label{alg:mirror}
\textbf{Input}: Environment simulator and parameter uncertainty set~$\attractset$ \\
\textbf{Params}: Convergence threshold $\epsilon$, num perturbations $O$ \\
\mbox{\textbf{Output}: Minimax regret--optimal agent mixed  strategy~$\mixedoptpolicy$} \nolinebreak
\begin{algorithmic}[1] 
\STATE Select an initial parameter setting $\attract_0 \in \attractset$ at random \label{line:initial-attract}\\
\STATE Compute baseline and heuristic strategies $\defpolicy_{B_1}, \defpolicy_{B_2}, \ldots$\label{line:initial-baseline}\\
\STATE $\attractset_0 = \{\attract_0\}$ \\
\STATE $\policyset_0 = \{\defpolicy_{B_1}, \defpolicy_{B_2}, \ldots\}$ \label{line:add-baselines} \\
\FOR {epoch $e = 1, 2, \ldots$}
\STATE $(\mixedpolicy_{e}, \mixedattract_{e}) = \textsc{ComputeMixedNash}(\policyset_{e-1}, \attractset_{e-1})$ \\
\STATE $\defpolicy_e = \textsc{AgentOracle}(\mixedattract_{e})$ \label{line:agent-br} \\
\STATE $(\attract_e, \altpolicy_e) = \textsc{NatureOracle}(\mixedpolicy_{e})$ \label{line:nature-br} \\
\IF {$\regret(\mixedpolicy_{e}, \attract_{e}) - \regret(\mixedpolicy_{e-1}, \mixedattract_{e-1}) \leq \epsilon$ and $\reward(\defpolicy_e, \mixedattract_{e}) - \reward(\mixedpolicy_{e-1}, \mixedattract_{e-1}) \leq \epsilon$}
\STATE \textbf{return} $\mixedpolicy_{e}$ \label{line:converge}
\ENDIF
\FOR {perturbation $o = 1, \ldots, O$} \label{line:perturb-start} 
\STATE perturb $\attract_e$ as $\attract^o_{e}$
\STATE $\defpolicy_{e}^o = \textsc{AgentOracle}(\attract_{e}^o)$ \label{line:perturb-end}
\ENDFOR
\STATE Compute expected returns as $\reward(\defpolicy_e, \attract)$ for all $\attract \in \attractset_{e-1}$ and $\reward(\defpolicy, \attract_e)$ for all $\defpolicy \in \policyset_{e-1}$ \label{line:compute-returns} \\
\STATE Compute max-regret game payoffs as Equation~\ref{eq:regret-game-payoff} \label{line:update-game} \\
\STATE $\attractset_{e} = \attractset_{e-1} \cup \{\attract_e, \attract_e^1, \ldots, \attract_e^O\}$ \label{line:update-nature-set}
\STATE $\policyset_{e} = \policyset_{e-1} \cup \{\defpolicy_e, \altpolicy_e, \defpolicy_e^1, \ldots, \defpolicy_e^O\}$ \label{line:update-agent-set}
\ENDFOR
\end{algorithmic}
\end{algorithm}

The full MIRROR procedure for minimax regret optimization using RL is given in Algorithm~\ref{alg:mirror} and visualized in Figure~\ref{fig:our-algorithm}.
The three necessary components are:
\begin{enumerate}
    \item \textbf{Agent oracle}: An RL algorithm that, given mixed strategy $\mixedattract_e$ as a distribution over $\attractset_e$, learns an optimal policy $\defpolicy_e$ for the defender to maximize reward in the known environment described by $\mixedattract_e$. 
    
    \item \textbf{Nature oracle}: An RL algorithm to compute an alternative policy $\altpolicy_e$ and new environment parameters $\attract_e$ given the current agent mixed strategy $\mixedpolicy_e$ over all policies $\policyset_e$. 
    The nature oracle's objective is to maximize regret: the difference between expected value of alternative policy $\altpolicy_e$ and the agent strategy $\mixedpolicy_e$. 
    
    
    Ideally, the alternative policy would be the optimal policy given environment parameters $\attract_e$, that is, $\altpolicy_e = \optpolicy(\attract_e)$. However, given that these RL approaches do not guarantee perfect policies, we must account for the imperfection in these oracles, which we discuss in Section~\ref{sec:param-perturbation}. 
    
    \item \textbf{Mixed Nash equilibrium solver}: A solver to compute a mixed Nash equilibrium for each player as a distribution over $\policyset_e$ for the agent and over $\attractset_e$ for nature in the max-regret game defined in Definition~\ref{def:regret-game}. 
\end{enumerate}



The MIRROR procedure would unfold as follows. We begin with arbitrary initial parameter values $\attract_0$ and baseline strategies (lines \ref{line:initial-attract}--\ref{line:add-baselines}).
The agent then learns a best-response defender policy $\defpolicy_1$ against these initial parameter values (line~\ref{line:agent-br}). 
Nature responds with $\attract_1$ (line~\ref{line:nature-br}). We update the payoff matrix in the max-regret game (lines~\ref{line:compute-returns}--\ref{line:update-game}), add the best response strategies $\defpolicy_e$ and $\attract_e$ to the strategy sets $\policyset_e$ and $\attractset_e$ for the agent and nature respectively (lines~\ref{line:update-nature-set}--\ref{line:update-agent-set}), and continue until convergence. Upon convergence (line~\ref{line:converge}), we reach an $\epsilon$-equilibrium in which neither player improves their payoff by more than $\epsilon$. In practice, for the sake of runtime, we cap number of iterations of double oracle to 10, a strategy also employed by \citet{lanctot2017unified}. We also include parameter perturbation (lines~\ref{line:perturb-start}--\ref{line:perturb-end}), which we discuss in Section~\ref{sec:param-perturbation}.

In many double oracle settings, the process of computing a best response is typically fast, as the problem is reduced to single-player optimization. However, the nature oracle is particularly challenging to implement due to our objective of minimax regret. Additionally, the imperfect nature of our oracles implies we are not guaranteed to find exact best strategies. We discuss our approaches below. 



\subsection{The Agent Oracle}

We want to find the best policy in a given environment setting. In our specific setting of poaching prevention, we consider deep deterministic policy gradient (DDPG) \citep{lillicrap2015continuous}. Policy gradient methods allow us to differentiate directly through a parameterized policy, making them well-suited to continuous state and action spaces, which we have. 
Note again that MIRROR is agnostic to the specific algorithm used. DDPG specifically is not necessary; technically, the approach need not be RL-based as long as it enables efficient computation of a best response strategy. 

We initialize the agent's strategy set $\policyset$ with the baseline algorithms, described in Section~\ref{sec:experiments}. Other heuristic strategies, based on expert knowledge from the rangers, could be added as part of the initialization. Hyperparameters used to implement DDPG for the agent oracle are 2 hidden layers of size 16 and 32 nodes, actor learning rate $10^{-4}$, and critic learning rate $10^{-3}$.



\subsection{The Nature Oracle}

\begin{algorithm}[tb]
\caption{Nature Oracle}
\label{alg:nature-oracle}
\textbf{Input}: Agent mixed strategy $\mixedpolicy \in \simplex(\policyset)$ \\
\textbf{Parameters}: Wake--sleep frequency $\sleep$, num episodes $J$ \\
\textbf{Output}: Nature best response environment parameters~$\attract$ and alternative policy~$\altpolicy$
\begin{algorithmic}[1] 
\STATE Randomly initialize $\attract$ and $\altpolicy$
\FOR{episode $j = 1, 2, \ldots, J$}
\STATE Sample agent policy $\defpolicy \sim \mixedpolicy$ \\
\FOR {timestep $t = 1, \ldots, T$}
\STATE \textbf{if} $j \bmod 2 \kappa = 0$ \textbf{then} \hspace{.5em} Unfreeze $\altpolicy$ and $\attract$ \label{line:wake-start} \\
\STATE \textbf{else if} $j \bmod \kappa = 0$ \textbf{then} \hspace{.5em} Freeze $\altpolicy$ parameters \\
\STATE \textbf{else} \hspace{.5em}  Freeze $\attract$ parameters \label{line:wake-end} \\
\STATE Update $\altpolicy$ and $\attract$ using gradient ascent to maximize regret: $\reward(\altpolicy, \attract) - \reward(\defpolicy, \attract)$
\ENDFOR
\ENDFOR
\STATE \textbf{return} $\attract$,  $\altpolicy$
\end{algorithmic}
\end{algorithm}


Learning the nature oracle is one of the key challenges. 
Our insight is that the nature oracle's task is to perform the same task as the agent oracle, combined with the (non-inconsequential) task of learning the optimal environment parameters, made difficult by the minimax regret criterion. The nature oracle may use a similar RL setup as the agent oracle, but we now face the challenging task of updating both the alternative policy $\altpolicy$ as well as the environment parameters $\attract$ --- and the setting of $\attract$ changes both the rewards of the policies $\defpolicy$ and $\altpolicy$. 


An initial approach might be to use two separate optimizers, one to train $\altpolicy$ and another to learn $\attract$. However, as the environment parameters $\attract$ and the alternative policy $\altpolicy$ are strongly correlated, optimizing them independently would lead to sub-optimal solutions.
Therefore, we integrate 
$\attract$ and $\altpolicy$ in the same actor and critic networks in DDPG and optimize the two together. 

Our approach for the nature oracle is given in Algorithm~\ref{alg:nature-oracle}. Similar to the agent oracle to learn a best response policy $\defpolicy$, we use policy gradient to learn the alternative policy $\altpolicy$, which enables us to take the derivative directly through the parameters of $\altpolicy$ and $\attract$ to perform gradient descent. Note that the input to the DDPG policy learner is not just the state $\state^{(t)} = (\effort^{(t-1)}, \wildlife^{(t-1)}, t)$ but also the attractiveness $\attract$: $(\attract, \effort^{(t-1)}, \wildlife^{(t-1)}, t)$. Ideally, we would incrementally change the parameters $\attract$, then optimally learn each time. But that would be very slow in practice, requiring full convergence of DDPG to train $\altpolicy$ at every step. We compromise by adopting a wake--sleep procedure \citep{hinton1995wake} where we alternately update only $\altpolicy$, only $\attract$, or both $\altpolicy$ and $\attract$ together. We describe the procedure in lines \ref{line:wake-start}--\ref{line:wake-end} of Algorithm~\ref{alg:nature-oracle}, were $\sleep$ is a parameter controlling the frequency of updates between $\attract$ and $\altpolicy$.

\subsection{Mixed Nash Equilibrium Solver}

We solve for the mixed Nash equilibrium in the max-regret game with the support enumeration algorithm \citep{roughgarden2010algorithmic}, a solution approach based on linear programming, using the Nashpy implementation \citep{knight2018nashpy}. There may be multiple mixed Nash equilibria, but given that the game is zero-sum, we may take any one of them as we discuss in Section~\ref{sec:convergence}.

\subsection{Parameter Perturbation}
\label{sec:param-perturbation}

Ideally, the learned alternative policy would be the optimal policy given environment parameters $\attract$, that is, $\altpolicy = \optpolicy(\attract)$. However, the RL approaches do not guarantee perfect policies.
With RL oracles, we must consider the question: what to do when the oracles (inevitably) fail to find the optimal policy?
Empirically, we observe that for a given environment parameter setting $\attract$, the policy $\defpolicy$ learned by DDPG occasionally yields a reward $\reward(\defpolicy, \attract)$ that is surpassed by another policy $\defpolicy'$ trained on a different parameter setting $\attract'$, with $\reward(\defpolicy, \attract) < \reward(\defpolicy', \attract)$. So clearly the defender oracle is not guaranteed to produce a best response for a given nature strategy. 

Inspired by this observation, we make parameter perturbation a key feature of our approach (Algorithm~\ref{alg:mirror} lines \ref{line:perturb-start}--\ref{line:perturb-end}), inspired by reward randomization which has been successful in RL \citep{tang2021discovering,wang2020reinforcement}. In doing so, we leverage the property that, in theory, any valid policy can be added to the set of agent strategies $\policyset_e$. So we include all of the best responses to perturbed strategies by the nature oracle (see Figure~\ref{fig:our-algorithm} for an illustration), which enables us to be more thorough in looking for an optimal policy $\optpolicy$ for each parameter setting as well as find the defender best response. In that way, the double oracle serves a role similar to an ensemble in practice. 

Parameter perturbation is grounded in three key insights. First, the oracles may be imprecise, but evaluation is highly accurate (relative to the nature parameters). Second, we only have to evaluate reward once, then max regret can be computed with simple subtraction. So the step does not add much computational overhead. Third, adding more strategies to the strategy set comes at relatively low cost, as computing a mixed Nash equilibrium is relatively fast and scalable. Specifically, the problem of finding an equilibrium in a zero-sum game can be solved with linear programming, which has polynomial complexity in the size of the game tree. Thus, even if the oracles add many bad strategies, growing the payoff matrix, the computational penalty is low, and the solution quality penalty is zero as it never takes us further from a solution. 


\section{Convergence and Correctness}
\label{sec:convergence}

We prove that Algorithm~\ref{alg:mirror} converges to an $\epsilon$--minimax regret optimal strategy for the agent in a finite number of epochs if the uncertain Markov decision process (UMDP) satisfies a technical condition. The key idea of the proof is to exploit the equivalence of the value of the max-regret game and the minimax regret--optimal payoff in the UMDP. For these quantities to be equivalent, the max-regret game induced by the UMDP must satisfy a variant of the minimax theorem. Two broad classes of games that satisfy this condition are games with finite strategy spaces and continuous games; we show that the green security model of Section~\ref{sec:model} induces a continuous max-regret game.

We begin by observing that the lower value of the max-regret game is equal to the payoff of the minimax regret--optimal policy of the UMDP. Using Definition~\ref{def:regret-game}, we can write the \emph{lower value} of the max-regret game as:
\begin{align}
    \underline{v} := \max_{\mixedpolicy} \min_{\mixedattract} \left( \reward(\mixedpolicy,\mixedattract) - \reward(\mixedoptpolicy(\mixedattract),\mixedattract) \right)
\end{align}
which is algebraically equivalent to Equation~\ref{eq:minimax-regret} by the definition of the optimal mixed strategy $\mixedoptpolicy$ and rearrangement.

The connection between the lower value and the payoff received by the row player is well-known in games with finite strategy spaces as a consequence of the seminal minimax theorem \citep{neumann1928theorie}. However, no such result holds in general for games with infinite strategy spaces, where a mixed Nash equilibrium may fail to exist. For so-called continuous games, \citet{glicksberg1952a} shows that a mixed Nash equilibrium exists and the analogy to the minimax theorem holds.

\begin{definition}
A game is \emph{continuous} if the strategy space for each player is non-empty and compact and the utility function is continuous in strategy space.
\end{definition}

We formalize the required connection in Condition~\ref{cond:minimax}, which holds for both finite and continuous games.

\begin{cond}\label{cond:minimax}
Let $(\mixedpolicy, \mixedattract)$ be any $\epsilon$--mixed Nash equilibrium of the max-regret game and $\underline{v}$ be the lower value of the max-regret game. Then, $|\underline{v} - (\reward(\mixedpolicy, \mixedattract) - \reward(\mixedoptpolicy(\mixedattract), \mixedattract))| \leq \epsilon$.
\end{cond}

We show that our green security UMDP induces a continuous max-regret game.
\begin{proposition}
The max-regret game induced by the model of Section~\ref{sec:model} is continuous.
\end{proposition}
\begin{proof}
The defender's strategy space consists of an action in $[0,1]^N$ responding to each state. Because each action space is compact, the defender's strategy space is compact. Nature has a compact uncertainty space. Both are non-empty.

The defender's expected return in the max regret game (Definition~\ref{def:regret-game} and Equation~\ref{eq:r}) can be written as a composition of continuous functions: addition, multiplication, the max (required to compute max regret), the logistic function (required for Equation~\ref{eq:poacher-prob}), and exponentiation (Equation~\ref{eq:wildlife-response}). The composition of these functions is also continuous.
\end{proof}

We now prove the main technical lemma: that the defender oracle and the nature oracle calculate best responses in the max-regret game. Doing so implies that the mixed Nash equilibrium returned by Algorithm~\ref{alg:mirror} in the final subgame over finite strategy sets $(\policyset_e, \attractset_e)$ is an $\epsilon$--mixed Nash equilibrium of the entire max-regret game. This result allows us to apply Condition~\ref{cond:minimax}, showing equivalence of the lower value of the max-regret game and the minimax regret--optimal payoff.

\begin{figure*}
  \centering
  \input{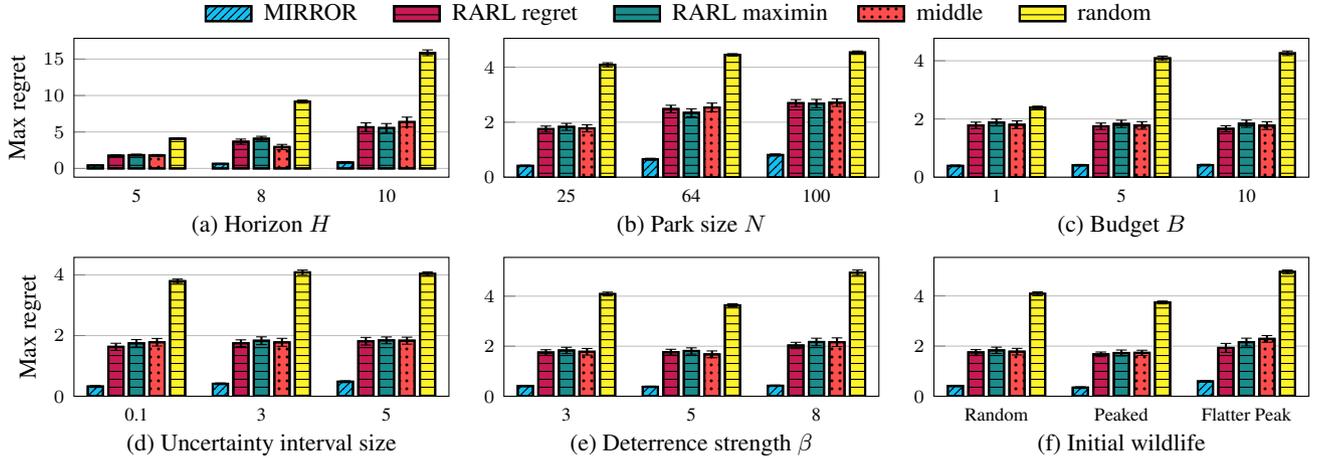}
  \caption{
  Comparing performance across varied settings, our MIRROR algorithm leads to the lowest max regret in all settings. We evaluate max regret by calculating the average reward difference between the selected policy and the optimal policy, with reward averaged over 100 episodes. We use as the default setting $H=5$, $N=25$, $B=5$, uncertainty interval $3$, $\pasteffortcoef=-5$, and random wildlife initialization. Standard error shown averaged over 30 trials.}
  \label{fig:performance}
\end{figure*}

\begin{lemma}
\label{lemma:MIRROR}
At epoch~$e$, policy $\defpolicy_e$ and environment parameters $\attract_e$ are best responses in the max-regret game to mixed strategies $\mixedattract_e$ and $\mixedpolicy_e$, respectively.
\end{lemma}
\begin{proof}
For the nature oracle, this is immediate because the reward of the nature oracle is exactly the payoff nature would receive in the max-regret game when playing against $\mixedpolicy_{e-1}$. For the agent oracle, the expected payoff of a strategy $\defpolicy$ against $\mixedattract_{e-1}$ in the max-regret game is $\mathbb{E}_{\attract \sim \mixedattract_{e-1}}[\reward(\defpolicy, \attract) - \reward(\optpolicy(\attract), \attract)]$. Because $\reward(\optpolicy(\attract), \attract)$ does not depend on $\defpolicy$, the policy that maximizes $\mathbb{E}_{\attract\sim\mixedattract_{e-1}}[\reward(\defpolicy, \attract)]$ also maximizes the agent's utility in the max-regret game. This quantity is exactly the reward for the agent oracle.
\end{proof}

\begin{theorem}
\label{thm:convergence}
If Condition~\ref{cond:minimax} holds and Algorithm~\ref{alg:mirror} converges, the agent mixed strategy returned by Algorithm~\ref{alg:mirror} achieves a minimax regret that is at most $\epsilon$ less than the minimax regret--optimal policy. If the max-regret game is either continuous with $\epsilon > 0$ or finite, Algorithm~\ref{alg:mirror} converges in a finite number of epochs.
\end{theorem}

\begin{proof}
Because the convergence condition for Algorithm~\ref{alg:mirror} is satisfied, $(\mixedpolicy_{e},\mixedattract_{e})$ is an $\epsilon$--mixed Nash equilibrium in the max-regret game by Lemma~\ref{lemma:MIRROR}. Applying Condition~\ref{cond:minimax} yields the result that the payoff of $\mixedpolicy_{e}$ is within $\epsilon$ of the minimax regret--optimal policy of the original UMDP.

If the max-regret game is finite, there are only finite number of strategies to add for each player and each strategy may be added only once---thus, Algorithm~\ref{alg:mirror} converges in finitely many epochs. If the max-regret game is continuous, Theorem 3.1 of \cite{adam2021double} guarantees convergence in finite epochs due to Lemma~\ref{lemma:MIRROR}.
\end{proof}

\section{Experiments}
\label{sec:experiments}

We conduct experiments using a simulator built from real poaching data from Queen Elizabeth National Park in Uganda, based on our analysis in Section~\ref{sec:deterrence}.
We consider robust patrol planning in the park with $N = 25$ to $100$ targets representing reasonably the area accessible from a patrol post. Each target is a $1 \times 1$~km region. 

We compare against the following four baselines. \textit{Middle} computes an optimal defender strategy assuming the true value of each parameter is the middle of the uncertainty interval. \textit{Random} takes a random strategy regardless of state. We apply the same parameter perturbations to the baselines as we do to the others and report the top-performing baseline variant. \textit{RARL maximin} uses robust adversarial learning \cite{pinto2017robust}, a robust approach optimizing for maximin reward (instead of minimax regret). We also add a variant we introduce of RARL, \textit{RARL regret}, which has a regret-maximizing adversary (instead of the reward-maximizing adversary typical in RARL) that leverages novel innovations of our nature oracle.
We evaluate performance of all algorithms in terms of maximum regret, computed using the augmented payoff matrix (with baselines and perturbed strategies) described in Section~\ref{sec:robust-planning}. The max regret is calculated by determining, for each parameter value, the defender strategy with the highest reward. In every experiment setting, we use the same strategy sets to compute max regret for all of the approaches shown.
Note that we would not expect any algorithm that optimizes for maximin reward to perform significantly better in terms of max regret than the middle strategy due to the regret criterion. 

Figure~\ref{fig:performance} shows the performance of our MIRROR algorithm compared to the baselines. Across variations of episode horizon, park size, deterrence strength, wildlife initial distributions, budget, and uncertainty interval size, MIRROR significantly reduces max regret. 
Deterrence strength changes the value of $\pasteffortcoef$ in Equation~\ref{eq:poacher-prob} to reveal the potential effectiveness of our actions. The wildlife initializations options are a uniform random distribution, a peaked Gaussian kernel (representing a core animal sanctuary in the park center), and a flatter Gaussian kernel (representing animals distributed more throughout the park, although more concentrated in the center). The uncertainty interval size restricts the maximum uncertainty range $\attracthigh{i} - \attractlow{i}$ for any target~$i$. 

One of the most notable strengths for MIRROR is shown in Figure~\ref{fig:performance}(a). As the episode horizon increases, thus the defender is tasked with planning longer-term sequences of decisions, MIRROR suffers only mildly more regret while the regret of the baseline strategies increases significantly. The scalability of MIRROR is evidenced in Figure~\ref{fig:performance}(b) as our relative performance holds when we consider larger-sized parks.  

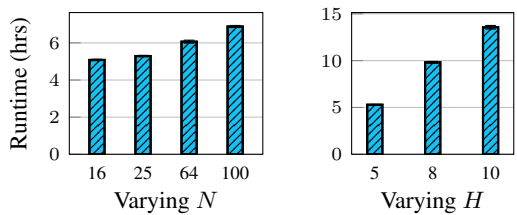
\begin{figure}
  \centering
  \begin{tikzpicture}
\pgfplotsset{
  width=.5\linewidth,
  height=0.42\linewidth,
  xtick pos=left,
  ytick pos=left,
  tick label style={font=\scriptsize},
  ymajorgrids=true,
  x tick label style={yshift=.7ex}, 
  xlabel shift=-2pt,
  xtick style={draw=none},
}

\begin{axis}[
  table/col sep=comma,
  ybar,
  bar width=.2cm,
  symbolic x coords={16, 25, 64, 100},
  xtick={data},
  xlabel={\footnotesize{Varying $N$}}, 
  ylabel={\footnotesize{Runtime (hrs)}},
  legend style={
    at={(1.05,1.26)},
    legend columns=3,
    font=\small,draw=none,fill=none},
  ymin=0, 
  enlarge x limits={0.2},
  error bars/y dir=both, 
  error bars/y explicit,  
  error bars/error bar style={color=black, thick},
]

\addplot+[cyan!70!white, draw=black, thick, area legend,
postaction={pattern=north east lines,}
] table [x=n, y=runtime_avg, y error=runtime_stderr]{data/runtime_n.csv};

\end{axis}

\begin{axis}[
  width=0.45\linewidth,
  at={(.45\linewidth, 0\linewidth)},
  table/col sep=comma,
  ybar,
  bar width=.2cm,
  symbolic x coords={5, 8, 10},
  xtick={data},
  xlabel={\footnotesize{Varying $H$}}, 
  legend style={
    at={(1.05,1.26)},
    legend columns=3,
    font=\small,draw=none,fill=none},
  ymin=0, 
  enlarge x limits={0.2},
  error bars/y dir=both, 
  error bars/y explicit,  
  error bars/error bar style={color=black, thick},
]

\addplot+[cyan!70!white, draw=black, thick, area legend,
postaction={pattern=north east lines,}
] table [x=horizon, y=runtime_avg, y error=runtime_stderr]{data/runtime_h.csv};
\end{axis}

\end{tikzpicture}
  \caption{The runtime of MIRROR  remains reasonable as we increase the park size~$N$ (which also increases the number of uncertain parameters) and horizon~$H$.}
  \label{fig:runtime}
\end{figure}

The runtime is shown in Figure~\ref{fig:runtime}, where we show that MIRROR is able to run in reasonable time as we scale to larger problem sizes, including in settings with 100 uncertain parameters ($N=100$). Rangers typically plan patrols once a month, so it is reasonable in practice to allot 5 to 15 hours of compute per month to plan. Tests were run on a cluster running CentOS with Intel(R) Xeon(R) CPU E5-2683 v4 @ 2.1 GHz with 16 GB RAM and 4 cores. 

Our strong empirical performance offers promise for 
effective real-world deployment for MIRROR. Uncertainty in the exact environment parameters is one of the most prominent challenges of sequential planning in the complex real-world setting of green security.

\section{Conclusion}
Our work is the first, across artificial intelligence and conservation biology literature, to show ranger patrols do deter poachers on real-world poaching data. Following this finding, we identify the problem of sequential planning for green security that is robust to parameter uncertainty following the minimax regret criterion, a problem that has not been studied in the literature. We address this challenge with our novel RL-based framework, MIRROR, which enables us to learn policies evaluated on minimax regret. We show the strength of MIRROR both theoretically, as it converges to an $\epsilon$--max regret optimal strategy in finite iterations, and empirically, as it leads to low-regret policies. We hope that our results inspire more work in green security based on our realistic adversary model and that our MIRROR framework is useful for future work on learning RL-policies that are optimal under minimax regret. 




\begin{acknowledgements} 

We are thankful to the Uganda Wildlife Authority for granting us access to incident data from Murchison Falls National Park. This work was supported in part by the Army Research Office (MURI W911NF1810208), NSF grant IIS-1850477, and IIS-2046640 (CAREER). Perrault and Chen were supported by the Center for Research on Computation and Society.
\end{acknowledgements}

\bibliography{short,ref}

\begin{thebibliography}{49}
\providecommand{\natexlab}[1]{#1}
\providecommand{\url}[1]{\texttt{#1}}
\expandafter\ifx\csname urlstyle\endcsname\relax
  \providecommand{\doi}[1]{doi: #1}\else
  \providecommand{\doi}{doi: \begingroup \urlstyle{rm}\Url}\fi

\bibitem[Adam et~al.(2021)Adam, Hor{\v{c}}{\'\i}k, Kasl, and
  Kroupa]{adam2021double}
Luk{\'a}{\v{s}} Adam, Rostislav Hor{\v{c}}{\'\i}k, Tom{\'a}{\v{s}} Kasl, and
  Tom{\'a}{\v{s}} Kroupa.
\newblock Double oracle algorithm for computing equilibria in continuous games.
\newblock In \emph{Proc.\ of AAAI-21}, 2021.

\bibitem[Basilico et~al.(2012)Basilico, Gatti, and
  Amigoni]{basilico2012patrolling}
Nicola Basilico, Nicola Gatti, and Francesco Amigoni.
\newblock Patrolling security games: Definition and algorithms for solving
  large instances with single patroller and single intruder.
\newblock \emph{Artificial Intelligence}, 184:\penalty0 78--123, 2012.

\bibitem[Bosansky et~al.(2014)Bosansky, Kiekintveld, Lisy, and
  Pechoucek]{bosansky2014exact}
Branislav Bosansky, Christopher Kiekintveld, Viliam Lisy, and Michal Pechoucek.
\newblock An exact double-oracle algorithm for zero-sum extensive-form games
  with imperfect information.
\newblock \emph{Journal of Artificial Intelligence Research}, 51:\penalty0
  829--866, 2014.

\bibitem[Bo{\v{s}}ansk{\`y} et~al.(2016)Bo{\v{s}}ansk{\`y}, Lis{\`y}, Lanctot,
  {\v{C}}erm{\'a}k, and Winands]{bovsansky2016algorithms}
Branislav Bo{\v{s}}ansk{\`y}, Viliam Lis{\`y}, Marc Lanctot, Ji{\v{r}}{\'\i}
  {\v{C}}erm{\'a}k, and Mark~HM Winands.
\newblock Algorithms for computing strategies in two-player simultaneous move
  games.
\newblock \emph{Artificial Intelligence}, 237:\penalty0 1--40, 2016.

\bibitem[Braziunas and Boutilier(2007)]{braziunas2007minimax}
Darius Braziunas and Craig Boutilier.
\newblock Minimax regret based elicitation of generalized additive utilities.
\newblock In \emph{Proc.\ of UAI-07}, 2007.

\bibitem[Dancer(2019)]{dancer2019evaluation}
Anthony Dancer.
\newblock \emph{On the evaluation, monitoring and management of law enforcement
  patrols in protected areas}.
\newblock PhD thesis, University College London, 2019.

\bibitem[Fang et~al.(2015)Fang, Stone, and Tambe]{fang2015security}
Fei Fang, Peter Stone, and Milind Tambe.
\newblock When security games go green: Designing defender strategies to
  prevent poaching and illegal fishing.
\newblock In \emph{Proc.\ of IJCAI-15}, 2015.

\bibitem[Fang et~al.(2016)Fang, Nguyen, Pickles, Lam, Clements, An, Singh,
  Tambe, and Lemieux]{fang2016deploying}
Fei Fang, Thanh~H Nguyen, Rob Pickles, Wai~Y Lam, Gopalasamy~R Clements, Bo~An,
  Amandeep Singh, Milind Tambe, and Andrew Lemieux.
\newblock Deploying {PAWS}: Field optimization of the {Protection Assistant for
  Wildlife Security}.
\newblock In \emph{Proc.\ of IAAI-16}, 2016.

\bibitem[Ford(2017)]{ford2017real}
Benjamin~John Ford.
\newblock \emph{Real-world evaluation and deployment of wildlife crime
  prediction models}.
\newblock PhD thesis, University of Southern California, 2017.

\bibitem[Gilbert and Spanjaard(2017)]{gilbert2017double}
Hugo Gilbert and Olivier Spanjaard.
\newblock A double oracle approach to minmax regret optimization problems with
  interval data.
\newblock \emph{European Journal of Operational Research}, 262\penalty0
  (3):\penalty0 929--943, 2017.

\bibitem[Glicksberg(1952)]{glicksberg1952a}
I.~L. Glicksberg.
\newblock A further generalization of the {Kakutani} fixed point theorem, with
  application to {Nash} equilibrium points.
\newblock In \emph{Proceedings of the American Mathematical Society}, pages
  170--174, 1952.

\bibitem[Gurumurthy et~al.(2018)Gurumurthy, Yu, Zhang, Jin, Li, Zhang, and
  Fang]{gurumurthy2018exploiting}
Swaminathan Gurumurthy, Lantao Yu, Chenyan Zhang, Yongchao Jin, Weiping Li,
  Xiaodong Zhang, and Fei Fang.
\newblock Exploiting data and human knowledge for predicting wildlife poaching.
\newblock In \emph{Proc.\ of COMPASS-18}, pages 1--8, 2018.

\bibitem[Hinton et~al.(1995)Hinton, Dayan, Frey, and Neal]{hinton1995wake}
Geoffrey~E Hinton, Peter Dayan, Brendan~J Frey, and Radford~M Neal.
\newblock The ``wake-sleep'' algorithm for unsupervised neural networks.
\newblock \emph{Science}, 268\penalty0 (5214):\penalty0 1158--1161, 1995.

\bibitem[Jain et~al.(2011)Jain, Korzhyk, Van{\v{e}}k, Conitzer,
  P{\v{e}}chou{\v{c}}ek, and Tambe]{jain2011double}
Manish Jain, Dmytro Korzhyk, Ond{\v{r}}ej Van{\v{e}}k, Vincent Conitzer, Michal
  P{\v{e}}chou{\v{c}}ek, and Milind Tambe.
\newblock A double oracle algorithm for zero-sum security games on graphs.
\newblock In \emph{Proc.\ of AAMAS-11}, pages 327--334, 2011.

\bibitem[Kar et~al.(2017)Kar, Ford, Gholami, Fang, Plumptre, Tambe, Driciru,
  Wanyama, Rwetsiba, Nsubaga, et~al.]{kar2017cloudy}
Debarun Kar, Benjamin Ford, Shahrzad Gholami, Fei Fang, Andrew Plumptre, Milind
  Tambe, Margaret Driciru, Fred Wanyama, Aggrey Rwetsiba, Mustapha Nsubaga,
  et~al.
\newblock Cloudy with a chance of poaching: Adversary behavior modeling and
  forecasting with real-world poaching data.
\newblock In \emph{Proc.\ of AAMAS-16}, 2017.

\bibitem[Knight and Campbell(2018)]{knight2018nashpy}
Vincent Knight and James Campbell.
\newblock Nashpy: A python library for the computation of {Nash} equilibria.
\newblock \emph{Journal of Open Source Software}, 3\penalty0 (30):\penalty0
  904, 2018.

\bibitem[Korzhyk et~al.(2010)Korzhyk, Conitzer, and
  Parr]{korzhyk2010complexity}
Dmytro Korzhyk, Vincent Conitzer, and Ronald Parr.
\newblock Complexity of computing optimal {Stackelberg} strategies in security
  resource allocation games.
\newblock In \emph{Proc.\ of AAAI-10}, volume~24, 2010.

\bibitem[Kouvelis and Yu(2013)]{kouvelis2013robust}
Panos Kouvelis and Gang Yu.
\newblock \emph{Robust discrete optimization and its applications}, volume~14.
\newblock Springer Science \& Business Media, 2013.

\bibitem[Lanctot et~al.(2017)Lanctot, Zambaldi, Gruslys, Lazaridou, Tuyls,
  P{\'e}rolat, Silver, and Graepel]{lanctot2017unified}
Marc Lanctot, Vinicius Zambaldi, Audrunas Gruslys, Angeliki Lazaridou, Karl
  Tuyls, Julien P{\'e}rolat, David Silver, and Thore Graepel.
\newblock A unified game-theoretic approach to multiagent reinforcement
  learning.
\newblock \emph{Proc.\ of NeurIPS-17}, 30:\penalty0 4190--4203, 2017.

\bibitem[Levitt(1998)]{levitt1998increased}
Steven~D Levitt.
\newblock Why do increased arrest rates appear to reduce crime: deterrence,
  incapacitation, or measurement error?
\newblock \emph{Economic inquiry}, 36\penalty0 (3):\penalty0 353--372, 1998.

\bibitem[Li et~al.(2019)Li, Wu, Cui, Dong, Fang, and Russell]{li2019robust}
Shihui Li, Yi~Wu, Xinyue Cui, Honghua Dong, Fei Fang, and Stuart Russell.
\newblock Robust multi-agent reinforcement learning via minimax deep
  deterministic policy gradient.
\newblock In \emph{Proc.\ of AAAI-19}, volume~33, pages 4213--4220, 2019.

\bibitem[Lillicrap et~al.(2016)Lillicrap, Hunt, Pritzel, Heess, Erez, Tassa,
  Silver, and Wierstra]{lillicrap2015continuous}
Timothy~P. Lillicrap, Jonathan~J. Hunt, Alexander Pritzel, Nicolas Heess, Tom
  Erez, Yuval Tassa, David Silver, and Daan Wierstra.
\newblock Continuous control with deep reinforcement learning.
\newblock In \emph{Proc.\ of ICLR-16}, 2016.

\bibitem[Loomes and Sugden(1982)]{loomes1982regret}
Graham Loomes and Robert Sugden.
\newblock Regret theory: An alternative theory of rational choice under
  uncertainty.
\newblock \emph{The Economic Journal}, 92\penalty0 (368):\penalty0 805--824,
  1982.

\bibitem[Marecki et~al.(2012)Marecki, Tesauro, and Segal]{marecki2012playing}
Janusz Marecki, Gerry Tesauro, and Richard Segal.
\newblock Playing repeated {Stackelberg} games with unknown opponents.
\newblock In \emph{Proc.\ of AAMAS-12}, pages 821--828, 2012.

\bibitem[McMahan et~al.(2003)McMahan, Gordon, and Blum]{mcmahan2003planning}
H~Brendan McMahan, Geoffrey~J Gordon, and Avrim Blum.
\newblock Planning in the presence of cost functions controlled by an
  adversary.
\newblock In \emph{Proc.\ of ICML-03}, pages 536--543, 2003.

\bibitem[Mnih et~al.(2015)Mnih, Kavukcuoglu, Silver, Rusu, Veness, Bellemare,
  Graves, Riedmiller, Fidjeland, Ostrovski, et~al.]{mnih2015human}
Volodymyr Mnih, Koray Kavukcuoglu, David Silver, Andrei~A Rusu, Joel Veness,
  Marc~G Bellemare, Alex Graves, Martin Riedmiller, Andreas~K Fidjeland, Georg
  Ostrovski, et~al.
\newblock Human-level control through deep reinforcement learning.
\newblock \emph{Nature}, 518\penalty0 (7540):\penalty0 529--533, 2015.

\bibitem[Moore et~al.(2018)Moore, Mulindahabi, Masozera, Nichols, Hines,
  Turikunkiko, and Oli]{moore2018ranger}
Jennifer~F Moore, Felix Mulindahabi, Michel~K Masozera, James~D Nichols,
  James~E Hines, Ezechiel Turikunkiko, and Madan~K Oli.
\newblock Are ranger patrols effective in reducing poaching-related threats
  within protected areas?
\newblock \emph{Journal of Applied Ecology}, 55\penalty0 (1):\penalty0 99--107,
  2018.

\bibitem[Nguyen et~al.(2016)Nguyen, Sinha, Gholami, Plumptre, Joppa, Tambe,
  Driciru, Wanyama, Rwetsiba, Critchlow, et~al.]{nguyen2016capture}
Thanh~H Nguyen, Arunesh Sinha, Shahrzad Gholami, Andrew Plumptre, Lucas Joppa,
  Milind Tambe, Margaret Driciru, Fred Wanyama, Aggrey Rwetsiba, Rob Critchlow,
  et~al.
\newblock {CAPTURE}: A new predictive anti-poaching tool for wildlife
  protection.
\newblock In \emph{Proc.\ of AAMAS-16}, pages 767--775, 2016.

\bibitem[Nguyen et~al.(2014)Nguyen, Yadav, An, Tambe, and
  Boutilier]{nguyen2014regret}
Thanh~Hong Nguyen, Amulya Yadav, Bo~An, Milind Tambe, and Craig Boutilier.
\newblock Regret-based optimization and preference elicitation for
  {Stackelberg} security games with uncertainty.
\newblock In \emph{AAAI}, pages 756--762, 2014.

\bibitem[Pan et~al.(2019)Pan, Seita, Gao, and Canny]{pan2019risk}
Xinlei Pan, Daniel Seita, Yang Gao, and John Canny.
\newblock Risk averse robust adversarial reinforcement learning.
\newblock In \emph{Proc.\ of ICRA-19}, pages 8522--8528. IEEE, 2019.

\bibitem[Pinto et~al.(2017)Pinto, Davidson, Sukthankar, and
  Gupta]{pinto2017robust}
Lerrel Pinto, James Davidson, Rahul Sukthankar, and Abhinav Gupta.
\newblock Robust adversarial reinforcement learning.
\newblock In \emph{Proc.\ of ICML-17}, 2017.

\bibitem[Regan and Boutilier(2009)]{regan2009regret}
Kevin Regan and Craig Boutilier.
\newblock Regret-based reward elicitation for {Markov} decision processes.
\newblock In \emph{Proc.\ of UAI-09}, 2009.

\bibitem[Rigter et~al.(2021)Rigter, Lacerda, and Hawes]{rigter2021minimax}
Marc Rigter, Bruno Lacerda, and Nick Hawes.
\newblock Minimax regret optimisation for robust planning in uncertain {Markov}
  decision processes.
\newblock In \emph{Proc.\ of AAAI-21}, 2021.

\bibitem[Roughgarden(2010)]{roughgarden2010algorithmic}
Tim Roughgarden.
\newblock Algorithmic game theory.
\newblock \emph{Communications of the ACM}, 53\penalty0 (7):\penalty0 78--86,
  2010.

\bibitem[Savage(1951)]{savage1951theory}
Leonard~J Savage.
\newblock The theory of statistical decision.
\newblock \emph{Journal of the American Statistical Association}, 46\penalty0
  (253):\penalty0 55--67, 1951.

\bibitem[Sessa et~al.(2020)Sessa, Bogunovic, Kamgarpour, and
  Krause]{sessa2020learning}
Pier~Giuseppe Sessa, Ilija Bogunovic, Maryam Kamgarpour, and Andreas Krause.
\newblock Learning to play sequential games versus unknown opponents.
\newblock In \emph{Proc.\ of NeurIPS-20}, 2020.

\bibitem[Tang et~al.(2021)Tang, Yu, Chen, Xu, Wang, Fang, Du, Wang, and
  Wu]{tang2021discovering}
Zhenggang Tang, Chao Yu, Boyuan Chen, Huazhe Xu, Xiaolong Wang, Fei Fang,
  Simon~Shaolei Du, Yu~Wang, and Yi~Wu.
\newblock Discovering diverse multi-agent strategic behavior via reward
  randomization.
\newblock In \emph{Proc.\ of ICLR-21}, 2021.

\bibitem[Venugopal et~al.(2021)Venugopal, Bondi, Kamarthi, Dholakia, Ravindran,
  and Tambe]{venugopal2021reinforcement}
Aravind Venugopal, Elizabeth Bondi, Harshavardhan Kamarthi, Keval Dholakia,
  Balaraman Ravindran, and Milind Tambe.
\newblock Reinforcement learning for unified allocation and patrolling in
  signaling games with uncertainty.
\newblock In \emph{Proc.\ of AAMAS-21}, 2021.

\bibitem[von Neumann(1928)]{neumann1928theorie}
John von Neumann.
\newblock Zur theorie der gesellschaftsspiele.
\newblock \emph{Mathematische annalen}, 100\penalty0 (1):\penalty0 295--320,
  1928.

\bibitem[Wang et~al.(2020)Wang, Liu, and Li]{wang2020reinforcement}
Jingkang Wang, Yang Liu, and Bo~Li.
\newblock Reinforcement learning with perturbed rewards.
\newblock In \emph{Proc.\ of AAAI-20}, 2020.

\bibitem[Wang and Boutilier(2003)]{wang2003incremental}
Tianhan Wang and Craig Boutilier.
\newblock Incremental utility elicitation with the minimax regret decision
  criterion.
\newblock In \emph{Proc.\ of IJCAI-03}, volume~3, pages 309--316, 2003.

\bibitem[Wang et~al.(2019)Wang, Shi, Yu, Wu, Singh, Joppa, and
  Fang]{wang2019deep}
Yufei Wang, Zheyuan~Ryan Shi, Lantao Yu, Yi~Wu, Rohit Singh, Lucas Joppa, and
  Fei Fang.
\newblock Deep reinforcement learning for green security games with real-time
  information.
\newblock In \emph{Proc.\ of AAAI-19}, volume~33, pages 1401--1408, 2019.

\bibitem[Xu et~al.(2016)Xu, Tran-Thanh, and Jennings]{xu2016playing}
Haifeng Xu, Long Tran-Thanh, and Nicholas~R Jennings.
\newblock Playing repeated security games with no prior knowledge.
\newblock In \emph{Proc.\ of AAMAS-16}, pages 104--112, 2016.

\bibitem[Xu et~al.(2017)Xu, Ford, Fang, Dilkina, Plumptre, Tambe, Driciru,
  Wanyama, Rwetsiba, Nsubaga, et~al.]{xu2017optimal}
Haifeng Xu, Benjamin Ford, Fei Fang, Bistra Dilkina, Andrew Plumptre, Milind
  Tambe, Margaret Driciru, Fred Wanyama, Aggrey Rwetsiba, Mustapha Nsubaga,
  et~al.
\newblock Optimal patrol planning for green security games with black-box
  attackers.
\newblock In \emph{Proc.\ of GameSec-17}, pages 458--477. Springer, 2017.

\bibitem[Xu et~al.(2020)Xu, Gholami, Carthy, Dilkina, Plumptre, Tambe, Singh,
  Nsubuga, Mabonga, Driciru, et~al.]{xu2020stay}
Lily Xu, Shahrzad Gholami, Sara~Mc Carthy, Bistra Dilkina, Andrew Plumptre,
  Milind Tambe, Rohit Singh, Mustapha Nsubuga, Joshua Mabonga, Margaret
  Driciru, et~al.
\newblock Stay ahead of poachers: Illegal wildlife poaching prediction and
  patrol planning under uncertainty with field test evaluations.
\newblock In \emph{Proc.\ of ICDE-20}, 2020.

\bibitem[Xu et~al.(2021)Xu, Bondi, Fang, Perrault, Wang, and Tambe]{xu2021dual}
Lily Xu, Elizabeth Bondi, Fei Fang, Andrew Perrault, Kai Wang, and Milind
  Tambe.
\newblock Dual-mandate patrols: Multi-armed bandits for green security.
\newblock In \emph{Proc.\ of AAAI-21}, 2021.

\bibitem[Yang et~al.(2014)Yang, Ford, Tambe, and Lemieux]{yang2014adaptive}
Rong Yang, Benjamin Ford, Milind Tambe, and Andrew Lemieux.
\newblock Adaptive resource allocation for wildlife protection against illegal
  poachers.
\newblock In \emph{Proc.\ of AAMAS-14}, pages 453--460. Citeseer, 2014.

\bibitem[Zhang et~al.(2020{\natexlab{a}})Zhang, Chen, Xiao, Li, Boning, and
  Hsieh]{zhang2020robustdeep}
Huan Zhang, Hongge Chen, Chaowei Xiao, Bo~Li, Duane Boning, and Cho-Jui Hsieh.
\newblock Robust deep reinforcement learning against adversarial perturbations
  on observations.
\newblock In \emph{Proc.\ of NeurIPS-20}, 2020{\natexlab{a}}.

\bibitem[Zhang et~al.(2020{\natexlab{b}})Zhang, Sun, Tao, Genc, Mallya, and
  Basar]{zhang2020robustmulti}
Kaiqing Zhang, Tao Sun, Yunzhe Tao, Sahika Genc, Sunil Mallya, and Tamer Basar.
\newblock Robust multi-agent reinforcement learning with model uncertainty.
\newblock In \emph{Proc.\ of NeurIPS-20}, 2020{\natexlab{b}}.

\end{thebibliography}

\end{document}